\documentclass[conference]{IEEEtran}
\IEEEoverridecommandlockouts
\usepackage{cite}
\usepackage{amsmath,amssymb,amsfonts,amsthm}
\newtheorem{theorem}{Theorem}[]
\newtheorem{assumption}{Assumption}[]
\newtheorem{lemma}{Lemma}[]

\newtheorem{corollary}{Corollary}[]
\newtheorem{remark}{Remark}[]

\usepackage{algorithm, algorithmic}

\usepackage{graphicx}
\usepackage{textcomp}
\usepackage{xcolor}

\usepackage{subfigure}
\usepackage{parskip}

\def\BibTeX{{\rm B\kern-.05em{\sc i\kern-.025em b}\kern-.08em
    T\kern-.1667em\lower.7ex\hbox{E}\kern-.125emX}}
\begin{document}

\title{Over-the-air Federated Policy Gradient}

\author{
    \IEEEauthorblockN{
        Huiwen Yang\IEEEauthorrefmark{1}, Lingying Huang\IEEEauthorrefmark{2},
        Subhrakanti Dey\IEEEauthorrefmark{3}, Ling Shi\IEEEauthorrefmark{1}
    }
    \IEEEauthorblockA{\IEEEauthorrefmark{1}\textit{The Hong Kong University of Science and Technology}, Hong Kong, China \\
    \IEEEauthorrefmark{2}\textit{Nanyang Technological University}, Singapore \\ \IEEEauthorrefmark{3}\textit{Uppsala University}, Uppsala, Sweden}
    \IEEEauthorblockA{{hyangbr@connect.ust.hk, lingying.huang@ntu.edu.sg, subhrakanti.dey@angstrom.uu.se, eesling@ust.hk}}
 
 \thanks{The work by H. Yang and L. Shi is supported by the Hong Kong RGC General Research Fund 16211622.}   
}


\maketitle

\begin{abstract}
    In recent years, over-the-air aggregation has been widely considered in large-scale distributed learning, optimization, and sensing.
    In this paper, we propose an over-the-air federated policy gradient algorithm, where all agents simultaneously broadcast an analog signal carrying local information to a common wireless channel, and a central controller uses the received aggregated waveform to update the policy parameters. 
    We investigate the effect of noise and channel distortion on the convergence of the proposed algorithm, and establish the complexities of communication and sampling for finding an $\epsilon$-approximate stationary point. Finally, we present some simulation results to show the effectiveness of the algorithm.
\end{abstract}

\begin{IEEEkeywords}
Over-the-air aggregation, policy gradient, federated reinforcement learning, convergence, linear speedup.
\end{IEEEkeywords}

\section{Introduction}
Reinforcement learning (RL) has gained a lot of attention in recent years. 
In a single-agent RL setting, an agent takes action in an environment and then receives corresponding rewards. The agent's objective is to learn a policy, which can maximize the long-term cumulative rewards, via such an interactive process. This sequential decision-making problem is generally modeled as a Markov decision process (MDP). The agent has no knowledge about the reward function and the transition of the Markov chain, and it can only make decisions based on its observations of states and rewards. There are many algorithms for solving such problems, e.g., Q-learning~\cite{watkins1992q}, policy gradient (PG)~\cite{sutton1999policy}, and actor-critic methods~\cite{konda1999actor}. 
These methods have been applied in many practical RL tasks such as robotics~\cite{kober2013reinforcement}, autonomous driving~\cite{shalev2016safe}, and video games~\cite{shao2019survey}. 

Among these methods, the PG method is a popular approach to dealing with large and continuous state-action spaces. It parameterizes the policy with a parameter $\boldsymbol{\theta}\in\mathbb{R}^d$. The objective function of the PG method is the expected cumulative reward which is also parameterized by the parameter $\boldsymbol{\theta}$, and is usually non-concave. The goal of the PG method is to find a stationary point of the objective function, where the gradient of the objective function is zero, using gradient-based algorithms. However, it is impractical to calculate the exact gradient of an expectation. As a result, stochastic gradient estimators, e.g., REINFORCE~\cite{williams1992simple} and G(PO)MDP~\cite{baxter2001infinite} are used to approximate the gradient. To utilize such stochastic gradient estimators, an agent needs to sample different trajectories of the Markov chain. 
To further reduce the variance introduced by the approximation and speed up the convergence, SVRPG~\cite{papini2018stochastic} was proposed.
Moreover, there are other efforts made to handle continuous actions~\cite{silver2014deterministic}, utilize deep neural networks~\cite{lillicrap2015continuous}, etc.

Generally, the larger the state-action space is, the more trajectory samples are required to achieve the same level of accuracy. To deal with this problem, federated RL utilizes parallel sampling and computation with multiple agents, which are coordinated by a central server~\cite{nadiger2019federated, qi2021federated, xu2021multiagent, khodadadian2022federated}. One of the expectations of federated RL is to accelerate the overall convergence rate w.r.t the number of agents. However, in the federated RL paradigm, all agents need to communicate with the central server in each iteration. That is to say, employing more agents can achieve faster convergence, but raise the communication complexity as well. To reduce the communication cost, Chen~et~al.~\cite{chen2021communication} proposed an event-triggered mechanism for federated policy gradient. However, with a huge number of agents, the event-triggered mechanism still fails due to communication bottleneck. 

\textcolor{black}{In traditional wireless networks, each device typically sends its data individually, which relies on multiple access techniques, e.g., time division multiple access (TDMA) and frequency division multiple access (FDMA), to avoid interference among the signals sent by distinct devices. However, in many task-oriented applications, e.g., distributed learning~\cite{yang2020federated, amiri2020federated, zhu2020one}, sensing~\cite{li2019wirelessly, liu2023over}, and control~\cite{park2021optimized}, the exact value sent by each device is not essential, but the accumulated information of these values, such as the summation of agents' local gradients, plays an important role. 
Recently, over-the-air aggregation has been considered a candidate for task-oriented applications~\cite{zhu2021over}.
Utilizing the analog transmission technique, over-the-air aggregation can aggregate data from multiple wireless devices directly at the receiver side by allowing all devices to broadcast to the same wireless channel.
This approach can significantly reduce the amount of wireless transmission required, leading to higher communication efficiency and network capacity,  and lower energy consumption and latency.}
In the literature, over-the-air aggregation has been considered to facilitate large-scale federated RL. In~\cite{krouka2021communication}, over-the-air aggregation was adopted to facilitate an ADMM-based federated RL algorithm. In~\cite{dal2023over}, the effect of distortion and noise induced by over-the-air aggregation was studied for temporal difference algorithms.

In this paper, we propose an over-the-air federated policy gradient method. To exploit the superposition property of wireless channels, all agents are allowed to simultaneously broadcast an analog signal to the same frequency block. The central server uses the received accumulated waveform to update the policy parameters. The contributions of this paper are three-fold: 
\begin{itemize}
    \item We investigate the effect of noise and channel distortion on the convergence rate of over-the-air federated policy gradient; 
    \item We prove that the proposed algorithm can achieve a linear convergence speedup w.r.t. the number of agents; 
    \item We establish the complexity of communication and sampling for finding an $\epsilon$-approximate stationary point.
\end{itemize}

The remainder of this paper is organized as follows. Section \uppercase\expandafter{\romannumeral2} provides some preliminaries and the problem formulation. 
Section \uppercase\expandafter{\romannumeral3} presents the theoretical results.
Section \uppercase\expandafter{\romannumeral4} presents the simulation results. Section \uppercase\expandafter{\romannumeral5} concludes this paper and presents some future work.

\emph{Notations:} $\mathbb{R}$ is the set of real numbers and $\mathbb{R}^n$ is the $n$-dimensional Euclidean space.  $\mathcal{N}(m,\Sigma)$ represents the Gaussian distribution with $m$-mean vector and covariance matrix $\Sigma$. $\mathbb{E}[\cdot]$ is the expectation of a random variable. $\mathbb{P}(\cdot | \cdot)$ refers to conditional probability. 

\section{Problem Formulation}
\subsection{Federated Reinforcement Learning}
    The goal of federated RL is to efficiently solve a large-scale single-agent RL problem by learning in parallel with multiple workers (agents). 
    We consider there are $N$ agents, who interact with the same MDP characterized by a tuple $\left(\mathcal{S}, \mathcal{A}, \mathcal{P}, \mathcal{\gamma}, \mathcal{\rho}, l \right)$, where $\gamma\in(0,1)$ is the discounting factor, $\rho$ is the initial state distribution, $\mathcal{S}$ is the state space,  $\mathcal{A}$ is the state space, $l: \mathcal{S}\times\mathcal{A}\rightarrow\mathbb{R}$ is the loss function, and $\mathcal{P}$ is the space of the state transition kernels defined as mappings $\mathcal{S}\times\mathcal{A}\rightarrow \Delta(\mathcal{S})$. 

    The policy used by all agents is denoted by $\pi: \mathcal{S}\times\mathcal{A}\rightarrow[0,1]$, and $\pi(\boldsymbol{a}|\boldsymbol{s})$ is the probability that agents select action $\boldsymbol{a}$ at state $\boldsymbol{s}$. The objective of the RL problem is to minimize the expectation of the cumulative discounted loss over a time horizon of $T$, i.e.,
    \begin{equation}\label{eq:prl}
        \min_{\pi} J(\pi)\triangleq\mathbb{E}_{\mathcal{T}\sim\mathbb{P}(\cdot|\pi)}\left[\sum_{t=0}^T \gamma^t l(\boldsymbol{s}_t, \boldsymbol{a}_t)\right],
    \end{equation}
    where $\mathcal{T}=\{\boldsymbol{s}_0, \boldsymbol{a}_0, \boldsymbol{s}_1, \boldsymbol{a}_1, \ldots, \boldsymbol{s}_{T-1}, \boldsymbol{a}_{T-1}, \boldsymbol{s}_T\}$ is a trajectory of state-action pairs generated by the policy $\pi$. 

    Under the parallel RL setting, each agent aims to solve the problem~\eqref{eq:prl} and all agents exchange information via a central controller to collectively find a policy $\pi$. 

\subsection{Policy Gradient}
Suppose the policy $\pi$ is parameterized by $\boldsymbol{\boldsymbol{\theta}}\in\mathbb{R}^d$, which is denoted as $\pi(\cdot|\boldsymbol{s};\boldsymbol{\boldsymbol{\theta}})$. The probability distribution of a trajectory $\mathcal{T}$ and the expected discounted loss under the policy $\pi(\cdot|\boldsymbol{s};\boldsymbol{\theta})$ are denoted by $\mathbb{P}(\cdot|\boldsymbol{\theta})$ and $J(\boldsymbol{\theta})\triangleq\mathbb{E}_{\mathcal{T}\sim\mathbb{P}(\cdot|\boldsymbol{\theta})}\left[\sum_{t=0}^T \gamma^t l(\boldsymbol{s}_t, \boldsymbol{a}_t)\right]$, respectively. Then, the RL problem~\eqref{eq:prl} can be written as
\begin{equation}
    \min_{\boldsymbol{\theta}} J(\boldsymbol{\theta}).
\end{equation}
To solve this problem, one can use the gradient descent method $\boldsymbol{\theta}^{k+1}=\boldsymbol{\theta}^{k} - \alpha \nabla J(\boldsymbol{\theta}^{k})$, where $\alpha > 0$ is the step size and $\nabla J(\boldsymbol{\theta})$ is the policy gradient defined by 
\begin{equation}\label{eq:gradient}
\begin{split}
    \nabla J(\boldsymbol{\theta}) =\mathbb{E}_{\mathcal{T}\sim\mathbb{P}(\cdot|\boldsymbol{\theta})}\left[ \sum_{t=0}^T \phi_{\boldsymbol{\theta}}\left(t\right) \gamma^t l(\boldsymbol{s}_{t},  \boldsymbol{a}_{t}) \right],
\end{split}
\end{equation}
where $\phi_{\boldsymbol{\theta}}\left(t\right)\triangleq \sum_{\tau=0}^t \nabla\log\pi(\boldsymbol{a}_{\tau}|\boldsymbol{s}_{\tau};\boldsymbol{\theta})$.
However, it is difficult to compute the exact gradient~\eqref{eq:gradient} since the MDP model is unknown. As a result, stochastic estimators, e.g., REINFORCE~\cite{williams1992simple} and G(PO)MDP~\cite{baxter2001infinite}, are often used to approximate the policy gradient~\eqref{eq:gradient}. In this paper, we consider each agent $i$ obtains an approximated policy gradient via the following mini-batch G(PO)MDP gradient
\begin{equation}\label{eq:gpomdp}
\begin{split}
    \hat{\nabla}J_i(\boldsymbol{\theta})= \frac{1}{M} \sum_{m=0}^M\sum_{t=0}^T \phi_{\boldsymbol{\theta}}^{i,m}\left(t\right)\gamma^t l(\boldsymbol{s}_{t}^{i,m},  \boldsymbol{a}_{t}^{i,m}),
\end{split}
\end{equation}
where $M$ is the batch size, $\mathcal{T}_{i,m}:=\{\boldsymbol{s}_0^{i,m}, \boldsymbol{a}_0^{i,m}, \boldsymbol{s}_1^{i,m}, \boldsymbol{a}_1^{i,m}, \ldots, \boldsymbol{s}_{T-1}^{i,m}, \boldsymbol{a}_{T-1}^{i,m}, \boldsymbol{s}_{T}^{i,m}\}$ is the $m$-th sampled trajectory of agent $i$, and $\phi_{\boldsymbol{\theta}}^{i,m}\left(t\right)\triangleq \sum_{\tau=0}^t \nabla\log\pi(\boldsymbol{a}_{\tau}^{i,m}|\boldsymbol{s}_{\tau}^{i,m};\boldsymbol{\theta})$.

At time step $k$, each agent $i$ computes and uploads $\hat{\nabla} J_i(\boldsymbol{\theta}^k)$ to the central controller, and the central controller updates the global policy parameters via
\begin{equation}\label{eq:update1}
    \boldsymbol{\theta}^{k+1}=\boldsymbol{\theta}^{k} - \alpha \hat{\nabla} J(\boldsymbol{\theta}^{k}),
\end{equation}
where $\hat{\nabla}J(\boldsymbol{\theta}^{k}) = \frac{1}{N}\sum_{i=1}^N \hat{\nabla} J_i(\boldsymbol{\theta}^{k})$. Then the updated parameters will be sent back to all the agents for the subsequent sampling, estimating, and learning. The mini-batch policy gradient algorithm is summarized in Algorithm~1.

\subsection{Over-the-air Aggregation}

In this paper, we consider the transmission from the agents to the central controller adopting the over-the-air aggregation, i.e., each agent broadcasts an analog signal carrying the information $\hat{\nabla} J_i(\boldsymbol{\theta}^k)$ to a common wireless channel. With transmitter synchronization and phase compensation~\cite{sery2020analog}, the central controller at time step $k$ can receive the following signal
\begin{equation}\label{eq:direction}
    \boldsymbol{v}_k = \sum_{i=1}^N h_{i,k} \hat{\nabla} J_i(\boldsymbol{\theta}^k) + \boldsymbol{n}_k
\end{equation}
where $\boldsymbol{n}_k\sim\mathcal{N}(0,\sigma^2 I_d)$ is the additive white Gaussian noise, and $h_{i,k}$ is the channel gain suffered by agent $i$ at time step $k$, which is a random variable with mean $m_h$ and variance $\sigma_h^2$. Note that the channel gain can be $h_{i,k}=c_{i,k}p_{i,k}$, where $c_{i,k}$ is the actual channel gain and $p_{i,k}$ is the transmission power coefficient of agent $i$. 
We assume that $h_{i,k}, \forall i, k$ and $\boldsymbol{n}_k, \forall k$ are independent. As a result, the central controller will update the global parameters as follows:
\begin{equation}\label{eq:update}
    \boldsymbol{\theta}^{k+1}=\boldsymbol{\theta}^{k} - \alpha \frac{\boldsymbol{v}_k}{N}.
\end{equation}

The federated RL algorithm with the proposed over-the-air federated policy gradient is summarized in Algorithm~2. In the rest of this paper, we aim to prove the convergence of Algorithm~2.

\begin{algorithm}[]
    \label{algo:gpomdp}
    \caption{Federated Policy Gradient with Mini-batch G(PO)MDP}
    \begin{algorithmic}
        \REQUIRE Stepsize $\alpha>0$, $M$, and $T$.
        \STATE {Initialize $\boldsymbol{\theta}^0$}
        \FOR {$ k = 0, 1, \ldots, K$ }
            \STATE {Controller broadcasts $\boldsymbol{\theta}^k$ to all agents.}
            \FOR {$i=1,2,\ldots, n$}
                \STATE {Agent $i$ computes $\hat{\nabla}J_i(\boldsymbol{\theta}^{k})$ and uploads it to the central controller.}
            \ENDFOR
        \STATE {Controller calculates $\boldsymbol{\theta}^{k+1}$ according to~\eqref{eq:update1}.}
        \ENDFOR
    \end{algorithmic} 
\end{algorithm}

\begin{algorithm}[]
    \label{algo:ota}
    \caption{Over-the-air Federate Policy Gradient}
    \begin{algorithmic}
        \REQUIRE Stepsize $\alpha>0$, $M$, and $T$.
        \STATE {Initialize $\boldsymbol{\theta}^0$}
        \FOR {$ k = 0, 1, \ldots, K$ }
            \STATE {Controller broadcasts $\boldsymbol{\theta}^k$ to all agents.}
            \STATE {Each agent $i$ computes its gradient estimate $\hat{\nabla}J_i(\boldsymbol{\theta}^{k})$.
            \STATE All agents $i=1,2,\ldots,N$ simultaneously broadcast $\hat{\nabla}J_i(\boldsymbol{\theta}^{k}), i=1,2,\ldots,N$ to the common wireless channel.}
            \STATE {Controller calculates $\boldsymbol{\theta}^{k+1}$ according to~\eqref{eq:update}.}
        \ENDFOR
    \end{algorithmic} 
\end{algorithm}

\section{Main Results}
In this section, we provide the convergence analysis of the over-the-air federated policy gradient, which is established under the following assumptions.
\begin{assumption}\label{asm:lossbound}
   For all $\boldsymbol{s}\in\mathcal{S}$ and $\boldsymbol{a}\in\mathcal{A}$, the loss $l(\boldsymbol{s}, \boldsymbol{a})$ is bounded, i.e., $l(\boldsymbol{s}, \boldsymbol{a})\in[0, \overline{l}]$. Hence, for any $\boldsymbol{\theta}$, the cumulative loss $J(\boldsymbol{\theta})$ is bounded, i.e.,  $J(\boldsymbol{\theta})\in[0, \overline{l}/(1-\gamma)]$.
\end{assumption}
\begin{assumption}\label{asm:gradientbound}
    For all $\boldsymbol{\theta}\in\mathbb{R}^d$, $\boldsymbol{s}\in\mathcal{S}$ and $\boldsymbol{a}\in\mathcal{A}$, there exist constants $G, F > 0$ such that 
    $$\left\Vert \nabla\log\pi(\boldsymbol{a}|\boldsymbol{s};\boldsymbol{\theta}) \right\Vert \leq G \text{ and } \left|\frac{\partial^2}{\partial\theta_i\partial\theta_j}\log\pi(\boldsymbol{a}|\boldsymbol{s};\boldsymbol{\theta})\right| \leq F, \forall i,j,$$
    where $\theta_i$ denotes the $i$-th entries of $\boldsymbol{\theta}$.
 \end{assumption}
    
Before proving the convergence of Algorithm~2, we provide some lemmas that are the cornerstone for the convergence analysis. 
 \begin{lemma}~\cite[Lemma 5]{chen2021communication}
     Under Assumption~\ref{asm:lossbound} and Assumption~\ref{asm:gradientbound}, the cumulative loss $J(\boldsymbol{\theta})$ is $L$-smooth, i.e., for any $\boldsymbol{\theta}_1,\boldsymbol{\theta}_2\in\mathbb{R}^d$, it holds that $\left\Vert \nabla J(\boldsymbol{\theta}_1) - \nabla J(\boldsymbol{\theta}_2)\right\Vert\leq L\left\Vert \boldsymbol{\theta}_1-\boldsymbol{\theta}_2 \right\Vert$, where $L\triangleq \left(F+G^2+\frac{2\gamma G^2}{1-\gamma}\right)\frac{\gamma \overline{
     l}}{(1-\gamma)^2}$.
 \end{lemma}
 \begin{lemma}\label{lm1}
     Suppose $J(\boldsymbol{\theta})$ is $L$-smooth, and $\boldsymbol{\theta}^{k+1}$ is generated by~\eqref{eq:update}. 
     Then the following inequality holds
      \begin{equation}
         \begin{split}
             \mathcal{V}(\boldsymbol{\theta}^{k+1}) - \mathcal{V}(\boldsymbol{\theta}^{k}) 
             \leq&-\frac{\alpha m_h}{2}\left\Vert \nabla J(\boldsymbol{\theta}^{k})\right\Vert^2 \\&\quad+ \frac{\alpha m_h}{2} \left\Vert  \frac{\boldsymbol{v}_k}{m_h N}-\nabla J(\boldsymbol{\theta}^{k}) \right\Vert^2 
            \\&\quad + \left( \frac{L}{2}- \frac{1}{2\alpha m_h} \right)\left\Vert \boldsymbol{\theta}^{k+1}-\boldsymbol{\theta}^{k} \right\Vert^2,
         \end{split}
     \end{equation}
     where $\mathcal{V}(\boldsymbol{\theta}^{k})\triangleq J(\boldsymbol{\theta}^{k})- J(\boldsymbol{\theta}^*)$.
 \end{lemma}
 \begin{proof}
    See Appendix~A.
 \end{proof}

\begin{lemma}\label{lm2}
    Suppose Assumption~\ref{asm:lossbound} and Assumption~\ref{asm:gradientbound} hold. Then we have
    \begin{equation}
    \begin{split}
        &\mathbb{E}\left[\left\Vert\frac{\boldsymbol{v}_k}{m_h N} - \nabla J(\boldsymbol{\theta}^{k})\right\Vert^2 \right]
        \\\leq& \frac{\sigma^2}{N^2} +\frac{\sigma_h^2 V^2}{MNm_h^2}+\frac{M(\sigma_h^2-m_h^2)-\sigma_h^2}{MNm_h^2}\mathbb{E}\left[\left\Vert\nabla J(\boldsymbol{\theta}^k)\right\Vert^2\right],      
    \end{split}
    \end{equation}   
     where $V\triangleq\frac{G\overline{l}\gamma}{(1-\gamma)^2}$.
\end{lemma}
 \begin{proof}
    See Appendix~B.
 \end{proof}

Now, we are ready to state the main convergence results.
 \begin{theorem}\label{thm:convergence}
    Under Assumption~\ref{asm:lossbound} and Assumption~\ref{asm:gradientbound}, if $\sigma_h^2\leq (N+1)m_h^2$ and the stepsize $\alpha$ satisfies $\alpha\leq\frac{1}{m_h L}$, then
    \begin{equation}
         \begin{split}
             \frac{1}{K}\sum_{k=0}^{K-1} \mathbb{E}\left[\left\Vert\nabla J(\boldsymbol{\theta}^{k})\right\Vert^2 \right]
             \leq& \frac{2MN m_h\left(J(\boldsymbol{\theta}^{0})-J(\boldsymbol{\theta^*})\right)}{\alpha\Lambda_{N,M}^{\sigma_h,m_h} K} \\&+ \frac{M m_h^2 \sigma^2}{N\Lambda_{N,M}^{\sigma_h,m_h}} + \frac{\sigma_h^2 V^2}{\Lambda_{N,M}^{\sigma_h,m_h}},
         \end{split}
     \end{equation}
      where $\Lambda_{N,M}^{\sigma_h,m_h}\triangleq M(N+1)m_h^2-(M-1)\sigma_h^2$.
 \end{theorem}
  \begin{proof}
     See Appendix~\ref{app:thm1}.
 \end{proof}
 \begin{remark}
     Since the cumulative cost function is usually non-convex, it is difficult to find an optimal point. Therefore, our goal is to find a stationary point, which is common in the non-convex optimization literature~\cite{zhao2021distributed, xu2023gradient}. Similar to~\cite{papini2018stochastic, chen2021communication, zhao2021distributed, xu2023gradient}, we use the expected gradient norm to indicate the convergence to a stationary point in Theorem~\ref{thm:convergence}.
 \end{remark}
 
 Based on Theorem~\ref{thm:convergence}, we are able to establish the bounds of communication complexity and sampling complexity for deriving an $\epsilon$-approximate stationary point.
 \begin{corollary}
     Let $\epsilon>0$. If $\sigma_h^2\leq (N+1)m_h^2$, the stepsize $\alpha$ satisfies $\alpha\leq\frac{1}{m_h L}$, and the epoch size $K$ and the batch size $M$ are chosen to satisfy $K=\mathcal{O}\left(\frac{1}{\epsilon}\right)$, $N=\mathcal{O}\left(\frac{1}{\sqrt{\epsilon}}\right)$ and $M=\mathcal{O}\left(\frac{1}{N\epsilon}\right)$, then the generated $\{\boldsymbol{\theta}^k\}$ satisfy $\frac{1}{K}\sum_{k=0}^{K-1} \mathbb{E}\left[\left\Vert\nabla J(\boldsymbol{\theta}^{k})\right\Vert^2 \right]\leq \epsilon$. That is to say, by increasing $K$, $M$, and $N$, the proposed algorithm can converge to an $\epsilon$-approximate stationary point $\boldsymbol{\theta}$ satisfying that $\mathbb{E}\left[\left\Vert\nabla J(\boldsymbol{\theta})\right\Vert^2 \right]$ is arbitrarily small.
 \end{corollary}

\begin{remark}
     The condition $\sigma_h^2\leq(N+1)m_h^2$ puts a restriction on the channel statistics. Note that the more agents there are, the more easily the assumption can be satisfied. The convergence without this condition is shown in the following theorem, i.e., Theorem~\ref{thm:convergence2}, which reveals the effect induced by the random channel states.
\end{remark}

\begin{theorem}\label{thm:convergence2}
Under Assumption~\ref{asm:lossbound} and Assumption~\ref{asm:gradientbound}, if the stepsize $\alpha$ satisfies $\alpha\leq\frac{1}{m_h L}$, then    
 \begin{equation}\label{ieq2}
     \begin{split}
        &\frac{1}{K}\sum_{k=0}^{K-1} \mathbb{E}\left[\left\Vert \nabla J(\boldsymbol{\theta}^{k})\right\Vert^2 \right]
        \\\leq&\frac{2MN m_h \left(J(\boldsymbol{\theta}^{0})-J(\boldsymbol{\theta}^{*})\right)}{\alpha K(M(N+1)m_h^2+\sigma_h^2)}
        + \frac{M\sigma_h^2 V^2}{M(N+1)m_h^2+\sigma_h^2}
        \\&+ \frac{\sigma_h^2 V^2}{M(N+1)m_h^2+\sigma_h^2} 
         + \frac{ M m_h^2\sigma^2}{N(M(N+1)m_h^2+\sigma_h^2)}.
     \end{split}
 \end{equation}
\end{theorem}
 \begin{proof}
     See Appendix~\ref{app:thm2}.
 \end{proof}
\begin{remark}
    Note that the second term on the right-hand side of~\eqref{ieq2}, whose order is $\mathcal{O}\left(\frac{1}{N}\right)$, is induced by the variance of the channel states and can be degraded by employing more agents. This term can not be diminished by increasing the communication rounds $K$ and the batch size $M$ since the communication and sampling processes play no role in reducing the effect caused by the randomness of the channels. 
\end{remark}

\begin{figure}[]
    \centering
    \includegraphics[width=0.8\linewidth]{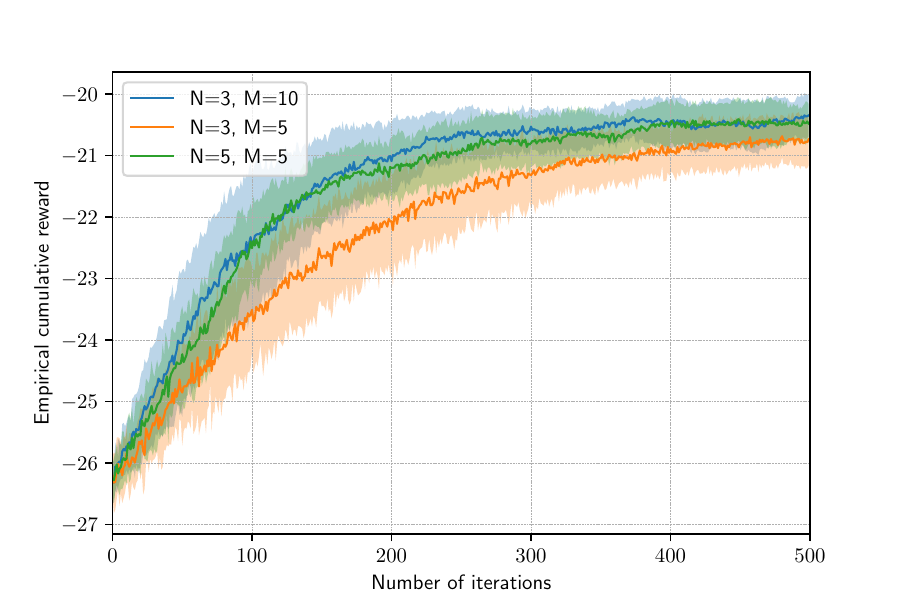}
    \caption{Empirical cumulative reward under Rayleigh channel ($\alpha=0.0001$).}
    \label{fig:reward_r}
    \centering
    \includegraphics[width=0.8\linewidth]{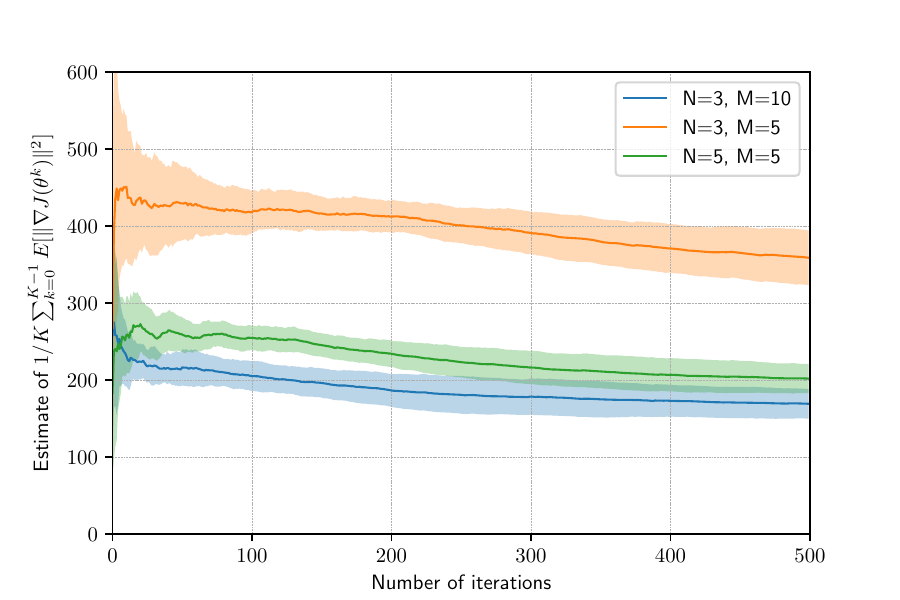}
    \caption{Estimate of $\frac{1}{K}\sum_{k=0}^{K-1} \mathbb{E}\left[\left\Vert\nabla J(\boldsymbol{\theta}^{k})\right\Vert^2 \right]$ under Rayleigh channel ($\alpha=0.0001$).}
    \label{fig:gradient_r}
    \centering
    \includegraphics[width=0.8\linewidth]{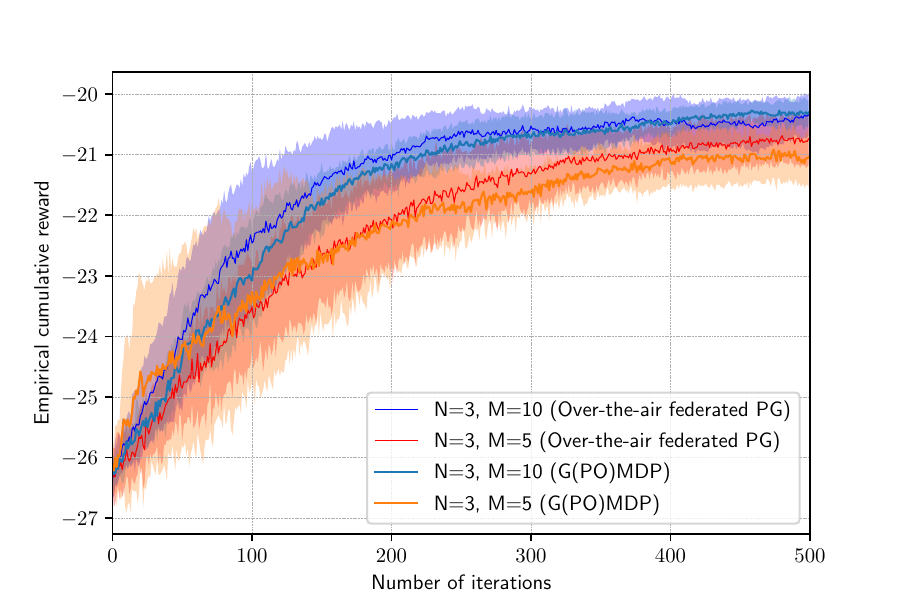}
    \caption{Comparison between vanilla G(PO)MDP and over-the-air federated PG under Rayleigh channel ($\alpha=0.0001$).}
    \label{fig:compare}
\end{figure}
\section{Simulation Results}
In this section, we evaluate the performance of the proposed algorithm on a classical RL environment provided by OpenAI~\cite{mordatch2017emergence}. The object of the agents is to cover a specific landmark. 
The agents know their own position and the landmark position, and are rewarded based on their distances to the landmark. \textcolor{black}{The state space of the MDP is $\mathcal{S}\triangleq\{(x, y, x', y')| x,y,x',y' \in \mathbb{R}\}$, where $(x,y)$ is the position of each agent and $(x',y')$ is the position of the landmark. The step reward is defined as $r(\boldsymbol{s}, \boldsymbol{a})=-l(\boldsymbol{s}, \boldsymbol{a}) \triangleq -\sqrt{(x-x')^2+(y-y')^2}$.} The action space for all agents is $\mathcal{A}=\{\textit{stay, left, right, up, down}\}$. Moreover, we set $\sigma^2=-60\textrm{dB}$, $\gamma=0.99$ and $T=20$.
The target policy is parameterized by a two-layer neural network, where the hidden layer contains 16 neurons with ReLU activation function, and the output layer adopts the softmax operator. The following simulations show the effect of batch size, number of agents, and channel statistics. For each setting, we present the simulation results of 20 Monte Carlo runs. 

We first consider that the channel gain $h_{i,k}$ follows Rayleigh distribution with $m_h=\sqrt{\frac{\pi}{2}}$ and $\sigma_h^2 = \frac{4-\pi}{2}$. It can be easily verified that $\sigma_h^2\leq (N+1)m_h^2$ holds for all $N$ and $M$. Fig.~\ref{fig:reward_r} and Fig.~\ref{fig:gradient_r} present the averaged empirical cumulative reward and estimate of $\frac{1}{K}\sum_{k=0}^{K-1} \mathbb{E}\left[\left\Vert\nabla J(\boldsymbol{\theta}^{k})\right\Vert^2 \right]$, respectively, with standard deviation. It can be seen that increasing both the batch size $M$ and the number of agents $N$ can improve the convergence performance. Specifically, Fig.~\ref{fig:gradient_r} verifies Theorem~\ref{thm:convergence} and shows that the proposed algorithm achieves linear speedup w.r.t. the number of agents $N$. 

\textcolor{black}{In Fig.~\ref{fig:compare}, we compare the performance of over-the-air federated PG with the vanilla G(PO)MDP. As shown in the figure, over-the-air federated PG converges within the same number of iterations as G(PO)MDP, which reveals that over-the-air federated PG, in such situations, achieves the same order of convergence rate as the vanilla G(PO)MDP. However, the distributed realization of G(PO)MDP requires that the policy gradients estimated by all agents are received exactly at the central controller, which usually relies on multiple access techniques such as TDMA and FDMA. In this sense, over-the-air federated PG needs fewer communication resources and hence can improve communication efficiency, especially when the number of agents is considerable.}
\begin{figure}[]
    \centering
    \includegraphics[width=0.8\linewidth]{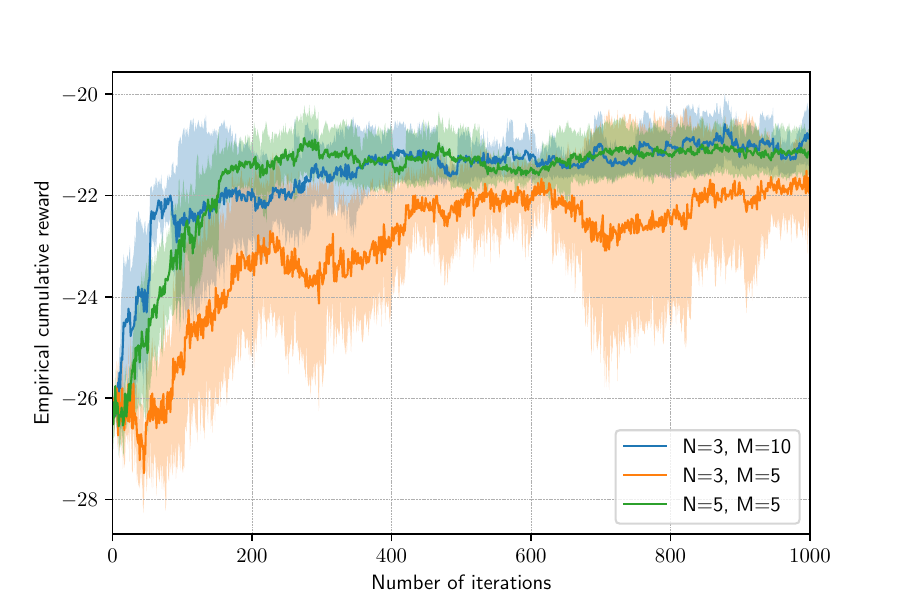}
    \caption{Empirical cumulative reward under Nakagami-$m$ channel ($\alpha=0.001$).}
    \label{fig:reward_g}
    \centering
    \includegraphics[width=0.8\linewidth]{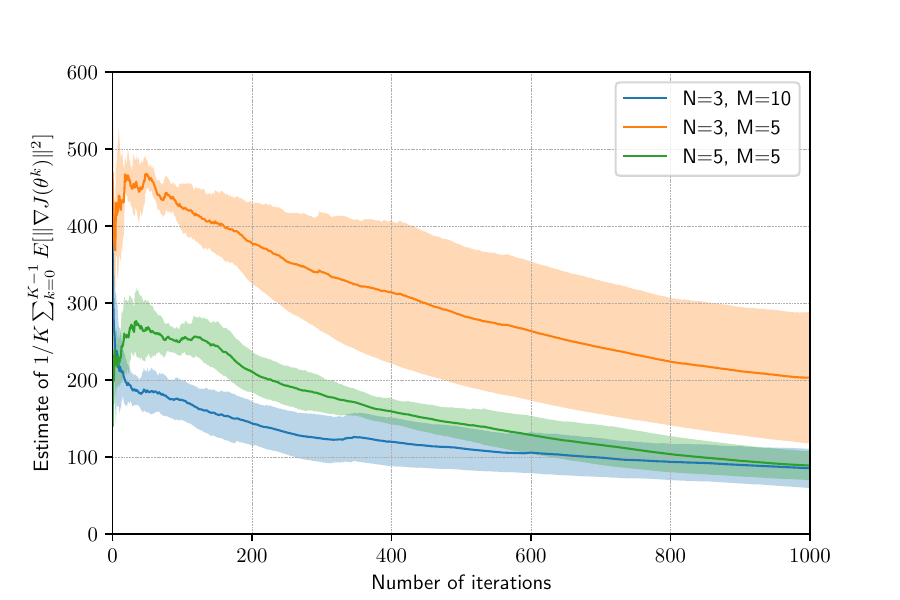}
    \caption{Estimate of $\frac{1}{K}\sum_{k=0}^{K-1} \mathbb{E}\left[\left\Vert\nabla J(\boldsymbol{\theta}^{k})\right\Vert^2 \right]$ under Nakagami-$m$ channel ($\alpha=0.001$).}
    \label{fig:gradient_g}
\end{figure}

Next, we consider Nakagami-$m$ channel with $m=0.1$ and $\Omega=1$, which satisfies $\sigma_h^2=10 m_h^2$. Fig.~\ref{fig:reward_g} and Fig.~\ref{fig:gradient_g} present the averaged empirical cumulative reward and estimate of $\frac{1}{K}\sum_{k=0}^{K-1} \mathbb{E}\left[\left\Vert\nabla J(\boldsymbol{\theta}^{k})\right\Vert^2 \right]$, respectively, with standard deviation. Fig.~\ref{fig:reward_g} shows that the convergence performance becomes worse than that under the Rayleigh channel. Comparing Fig~\ref{fig:gradient_g} with Fig.~\ref{fig:gradient_r}, it can be observed that the effectiveness of increasing the batch size $M$ is weakened, which verifies Theorem~\ref{thm:convergence2}.


\section{Conclusion}
In this paper, we propose an over-the-air federated policy gradient algorithm, which exploits the superposition property of wireless channels to cope with the bottleneck of traditional communication schemes. We investigate the effect of noise and channel distortion on the convergence. Moreover, we prove that for finding an $\epsilon$-approximate stationary point, the communication complexity of the proposed algorithm is $\mathcal{O}\left(\frac{1}{\epsilon}\right)$, and the sampling complexity is $\mathcal{O}\left(\frac{1}{N\epsilon}\right)$, which exhibits an N-fold linear speedup w.r.t. the number of agents $N$. Simulation results show the effectiveness of the algorithm.

In the future, we plan to investigate the scenario where agents learn in a collaborative manner, i.e., multiple agents interact with the common environment influenced by all the agents. Moreover, we will also consider a fully decentralized realization without the assistance of a central controller.

\appendices
\section{Proof of Lemma~\ref{lm1}}
    By~\eqref{eq:update}, we have
     \begin{equation}\label{ieq1}
     \begin{split}
        &\left\langle \nabla J(\boldsymbol{\theta}^{k}), \boldsymbol{\theta}^{k+1} - \boldsymbol{\theta}^{k} \right\rangle \\&= -\alpha m_h \left\langle \nabla J(\boldsymbol{\theta}^{k}),  \frac{\boldsymbol{v}_k}{m_h N} \right\rangle
        \\&= -\alpha m_h \left\langle \nabla J(\boldsymbol{\theta}^{k}),  \nabla J(\boldsymbol{\theta}^{k})-\nabla J(\boldsymbol{\theta}^{k})+\frac{\boldsymbol{v}_k}{m_h N} \right\rangle
        \\&= -\alpha m_h\left(\left\Vert \nabla J(\boldsymbol{\theta}^{k})\right\Vert^2 -\left\langle -\nabla J(\boldsymbol{\theta}^{k}),   \frac{\boldsymbol{v}_k}{m_h N}-\nabla J(\boldsymbol{\theta}^{k}) \right\rangle\right)
        \\&= -\alpha m_h\left\Vert \nabla J(\boldsymbol{\theta}^{k})\right\Vert^2 + \frac{\alpha m_h}{2} \left(\left\Vert\nabla J(\boldsymbol{\theta}^{k})\right\Vert^2 \right.\\&\left.\quad  + \left\Vert  \frac{\boldsymbol{v}_k}{m_h N}-\nabla J(\boldsymbol{\theta}^{k}) \right\Vert^2 - \left\Vert \frac{\boldsymbol{v}_k}{m_h N} \right\Vert^2 \right)
        \\&= -\frac{\alpha m_h}{2}\left\Vert \nabla J(\boldsymbol{\theta}^{k})\right\Vert^2 + \frac{\alpha m_h}{2} \left( \left\Vert  \frac{\boldsymbol{v}_k}{m_h N}-\nabla J(\boldsymbol{\theta}^{k}) \right\Vert^2 \right.\\&\left.\quad- \frac{1}{\alpha^2 m_h^2}\left\Vert \boldsymbol{\theta}^{k+1}-\boldsymbol{\theta}^{k} \right\Vert^2 \right),
     \end{split}
     \end{equation}
     where the fourth equality follows from $2\left\langle a,b\right\rangle=\Vert a\Vert^2+\Vert b\Vert^2-\Vert a-b\Vert^2$. 
     
    By the smoothness of $J$, we have
     \begin{equation}
         J(\boldsymbol{\theta}^{k+1}) - J(\boldsymbol{\theta}^{k}) \leq \left\langle \nabla J(\boldsymbol{\theta}^{k}), \boldsymbol{\theta}^{k+1} - \boldsymbol{\theta}^{k} \right\rangle + \frac{L}{2}\left\Vert \boldsymbol{\theta}^{k+1}-\boldsymbol{\theta}^{k} \right\Vert^2.
     \end{equation}
     Then, by~\eqref{ieq1}, we have
     \begin{equation}
         \begin{split}
             \mathcal{V}(\boldsymbol{\theta}^{k+1}) - \mathcal{V}(\boldsymbol{\theta}^{k}) =&
             J(\boldsymbol{\theta}^{k+1}) - J(\boldsymbol{\theta}^{k}) 
             \\\leq&-\frac{\alpha m_h}{2}\left\Vert \nabla J(\boldsymbol{\theta}^{k})\right\Vert^2 \\&\quad+ \frac{\alpha m_h}{2} \left\Vert  \frac{\boldsymbol{v}_k}{m_h N}-\nabla J(\boldsymbol{\theta}^{k}) \right\Vert^2 
            \\&\quad + \left( \frac{L}{2}- \frac{1}{2\alpha m_h} \right)\left\Vert \boldsymbol{\theta}^{k+1}-\boldsymbol{\theta}^{k} \right\Vert^2.
         \end{split}
     \end{equation}

\section{Proof of Lemma~\ref{lm2}}
    By Assumption~\ref{asm:lossbound} and Assumption~\ref{asm:gradientbound}, we have
    \begin{equation}
     \begin{split}
          &\mathbb{E}\left[\left.\left\Vert\hat{\nabla} J_i(\boldsymbol{\theta}^k)\right\Vert^2\right|\boldsymbol{\theta}^k\right] \\&=  \mathbb{E}\left[\left.\left\Vert \frac{1}{M} \sum_{m=1}^M\sum_{t=0}^T \phi_{\boldsymbol{\theta}^k}^{i,m}\left(t\right)\gamma^t l(\boldsymbol{s}_{t}^{i,m},  \boldsymbol{a}_{t}^{i,m})\right\Vert^2\right|\boldsymbol{\theta}^k\right]
          \\&\leq \sup_{\mathcal{T}\sim\mathbb{P}(\cdot|\boldsymbol{\theta}^k)} \left\Vert \frac{1}{M} \sum_{m=1}^M\sum_{t=0}^T \phi_{\boldsymbol{\theta}^k}^{i,m}\left(t\right)\gamma^t l(\boldsymbol{s}_{t}^{i,m},  \boldsymbol{a}_{t}^{i,m})\right\Vert^2
          \\&=\left\Vert\sum_{t=0}^T tG \gamma^t  \overline{l}\right\Vert^2 \leq \left\Vert G \overline{l}\sum_{t=0}^\infty t\gamma^t \right\Vert^2 = V^2,
     \end{split}
     \end{equation}
     where $V\triangleq\frac{G^2 \overline{l}^2\gamma^2}{(1-\gamma)^4}$.
     
     Define $\hat{\nabla} J_{i,m}\triangleq\sum_{t=0}^T \phi_{\boldsymbol{\theta}}^{i,m}\left(t\right)\gamma^t l(\boldsymbol{s}_{t}^{i,m},  \boldsymbol{a}_{t}^{i,m})$, $\boldsymbol{\delta}_k^{i,m} \triangleq\frac{h_{i,k}}{m_h}\hat{\nabla} J_{i,m}(\boldsymbol{\theta}^k)-\nabla J(\boldsymbol{\theta}^k)$ and the operator $\Phi(\boldsymbol{x}, \boldsymbol{y})\triangleq\mathbb{E}\left[\left.\left\langle \boldsymbol{x}, \boldsymbol{y} \right\rangle\right|\boldsymbol{\theta}^k\right]$, then we have
       \begin{equation}\label{ieq:norm}
         \begin{split}
            &\mathbb{E}\left[\left.\left\Vert\boldsymbol{\delta}_k^{i,m}\right\Vert^2\right|\boldsymbol{\theta}^k\right]
            \\&= \mathbb{E}\left[\frac{h_{i,k}^2}{m_h^2}\right]\mathbb{E}\left[\left.\left\Vert\hat{\nabla} J_{i,m}(\boldsymbol{\theta}^k)\right\Vert^2\right|\boldsymbol{\theta}^k\right]+\left\Vert\nabla J(\boldsymbol{\theta}^k)\right\Vert^2
            \\&\quad -2\mathbb{E}\left[\left.\left\langle\frac{h_{i,k}}{m_h}\hat{\nabla} J_{i,m}(\boldsymbol{\theta}^k), \nabla J(\boldsymbol{\theta}^k)\right\rangle\right|\boldsymbol{\theta}^k\right]
            \\&=\frac{\sigma_h^2}{m_h^2}\mathbb{E}\left[\left.\left\Vert\hat{\nabla} J_{i,m}(\boldsymbol{\theta}^k)\right\Vert^2\right|\boldsymbol{\theta}^k\right]+\left\Vert\nabla J(\boldsymbol{\theta}^k)\right\Vert^2
            \\&\quad -2\left\langle\mathbb{E}\left[\frac{h_{i,k}}{m_h}\right]\mathbb{E}\left[\left.\hat{\nabla} J_{i,m}(\boldsymbol{\theta}^k)\right|\boldsymbol{\theta}^k\right], \nabla J(\boldsymbol{\theta}^k)\right\rangle
            \\&=\frac{\sigma_h^2}{m_h^2}\mathbb{E}\left[\left.\left\Vert\hat{\nabla} J_{i,m}(\boldsymbol{\theta}^k)\right\Vert^2\right|\boldsymbol{\theta}^k\right]+\left\Vert\nabla J(\boldsymbol{\theta}^k)\right\Vert^2
            -2\left\Vert\nabla J(\boldsymbol{\theta}^k)\right\Vert^2
            \\&\leq \frac{\sigma_h^2V^2}{m_h^2}  - \left\Vert\nabla J(\boldsymbol{\theta}^k)\right\Vert^2,
         \end{split}
     \end{equation}
        and
     \begin{equation}\label{ieq:product}
         \begin{split}
             &\mathbb{E}\left[\left.\left\langle\boldsymbol{\delta}_k^{i,m}, \boldsymbol{\delta}_k^{i,n}\right\rangle \right|\boldsymbol{\theta}^k\right]
             \\&=\left\Vert\nabla J(\boldsymbol{\theta}^k)\right\Vert^2+\mathbb{E}\left[\frac{h_{i,k}^2}{m_h^2}\right]\mathbb{E}\left[\left.\left\langle\hat{\nabla} J_{i,m}(\boldsymbol{\theta}^k), \hat{\nabla} J_{i,n}(\boldsymbol{\theta}^k)\right\rangle\right|\boldsymbol{\theta}^k\right]
             \\& \quad - \mathbb{E}\left[\left.\left\langle\frac{h_{i,k}}{m_h}\left(\hat{\nabla} J_{i,m}(\boldsymbol{\theta}^k)+\hat{\nabla} J_{i,n}(\boldsymbol{\theta}^k)\right), \nabla J(\boldsymbol{\theta}^k)\right\rangle \right|\boldsymbol{\theta}^k\right]
             \\& = \left\Vert\nabla J(\boldsymbol{\theta}^k)\right\Vert^2 +\frac{\sigma_h^2}{m_h^2}\left\langle\mathbb{E}\left[\left.\hat{\nabla} J_{i,m}(\boldsymbol{\theta}^k) \right|\boldsymbol{\theta}^k\right], \mathbb{E}\left[\left.\hat{\nabla} J_{i,n}(\boldsymbol{\theta}^k) \right|\boldsymbol{\theta}^k\right]\right\rangle
             \\& \quad - \mathbb{E}\left[\frac{h_{i,k}}{m_h}\right]\left\langle\mathbb{E}\left[\left.\hat{\nabla} J_{i,m}(\boldsymbol{\theta}^k) \right|\boldsymbol{\theta}^k\right]+\mathbb{E}\left[\left.\hat{\nabla} J_{i,n}(\boldsymbol{\theta}^k) \right|\boldsymbol{\theta}^k\right], \nabla J(\boldsymbol{\theta}^k)\right\rangle
             \\&= \left\Vert\nabla J(\boldsymbol{\theta}^k)\right\Vert^2 +\frac{\sigma_h^2}{m_h^2}\left\Vert\nabla J(\boldsymbol{\theta}^k)\right\Vert^2 - 2\left\Vert\nabla J(\boldsymbol{\theta}^k)\right\Vert^2
             \\&= \left(\frac{\sigma_h^2}{m_h^2} - 1\right)\left\Vert\nabla J(\boldsymbol{\theta}^k)\right\Vert^2
             \leq \frac{\sigma_h^2 V^2}{m_h^2}  - \left\Vert\nabla J(\boldsymbol{\theta}^k)\right\Vert^2.
         \end{split}
     \end{equation}
     By \eqref{ieq:norm} and \eqref{ieq:product}, we have
     \begin{equation}\label{eq:variance}
     \begin{split}
        &\mathbb{E}\left[\left.\left\Vert\frac{\boldsymbol{v}_k}{m_h N} - \nabla J(\boldsymbol{\theta}^{k})\right\Vert^2  \right|\boldsymbol{\theta}^k\right]
        \\=& \frac{1}{N^2}\mathbb{E}\left[\left.\left\Vert\frac{1}{M} \sum_{i=1}^N \sum_{m=1}^M \boldsymbol{\delta}_k^{i,m} + \boldsymbol{n}_k \right\Vert^2 \right|\boldsymbol{\theta}^k\right]
        \\=& \frac{1}{M^2 N^2}\mathbb{E}\left[\left.\left\Vert \sum_{i=1}^N \sum_{m=1}^M \boldsymbol{\delta}_k^{i,m}\right\Vert^2 \right|\boldsymbol{\theta}^k\right] + \frac{1}{N^2}\mathbb{E}\left[\left\Vert \boldsymbol{n}_k \right\Vert^2\right]    
        \\=& \frac{\sigma^2}{N^2} + \frac{1}{M^2 N^2}\left(\sum_{i=1}^N \sum_{m=1}^M \mathbb{E}\left[\left.\left\Vert\boldsymbol{\delta}_k^{i,m}\right\Vert^2 \right|\boldsymbol{\theta}^k\right]\right.
         \\&\left. + \sum_{i=1}^N\sum_{m=1}^M\sum_{n\ne m}^M \mathbb{E}\left[\left.\left\langle\boldsymbol{\delta}_k^{i,m}, \boldsymbol{\delta}_k^{i,n}\right\rangle \right|\boldsymbol{\theta}^k\right]\right.
         \\&\left. + \sum_{i=1}^N\sum_{j\ne i}^N \sum_{m=1}^M\sum_{n=1}^M \mathbb{E}\left[\left.\left\langle\boldsymbol{\delta}_k^{i,m}, \boldsymbol{\delta}_k^{j,n}\right\rangle \right|\boldsymbol{\theta}^k\right]\right)
         \\=& \frac{\sigma^2}{N^2} + \frac{1}{M^2 N^2}\left(\sum_{i=1}^N \sum_{m=1}^M \mathbb{E}\left[\left.\left\Vert\boldsymbol{\delta}_k^{i,m}\right\Vert^2 \right|\boldsymbol{\theta}^k\right]\right.
         \\&\left. + \sum_{i=1}^N\sum_{m=1}^M\sum_{n\ne m}^M \mathbb{E}\left[\left.\left\langle\boldsymbol{\delta}_k^{i,m}, \boldsymbol{\delta}_k^{i,n}\right\rangle \right|\boldsymbol{\theta}^k \right]\right.
         \\&\left. + \sum_{i=1}^N\sum_{j\ne i}^N \sum_{m=1}^M\sum_{n=1}^M \left\langle\mathbb{E}\left[\left.\boldsymbol{\delta}_k^{i,m} \right|\boldsymbol{\theta}^k\right], \mathbb{E}\left[\left.\boldsymbol{\delta}_k^{j,n} \right|\boldsymbol{\theta}^k\right]\right\rangle\right)
        \\\leq& \frac{\sigma^2}{N^2} + \frac{1}{M^2 N^2}\left(NM\frac{\sigma_h^2 V^2}{m_h^2} - NM \left\Vert\nabla J(\boldsymbol{\theta}^k)\right\Vert^2 \right.
        \\&\left.+NM(M-1)\left(\frac{\sigma_h^2}{m_h^2}-1\right)\left\Vert\nabla J(\boldsymbol{\theta}^k)\right\Vert^2\right)
        \\=&\frac{\sigma^2}{N^2} +\frac{\sigma_h^2 V^2}{MNm_h^2}+\frac{M(\sigma_h^2-m_h^2)-\sigma_h^2}{MNm_h^2}\left\Vert\nabla J(\boldsymbol{\theta}^k)\right\Vert^2.
        \end{split}
    \end{equation}
     Then, we have
     \begin{equation}
     \begin{split}
          &\mathbb{E}\left[\left\Vert\frac{\boldsymbol{v}_k}{m_h N} - \nabla J(\boldsymbol{\theta}^{k})\right\Vert^2\right] \\&= \mathbb{E}\left[\mathbb{E}\left[\left.\left\Vert\frac{\boldsymbol{v}_k}{m_h N} - \nabla J(\boldsymbol{\theta}^{k})\right\Vert^2  \right|\boldsymbol{\theta}^k\right]\right]
          \\&\leq\frac{\sigma^2}{N^2} +\frac{\sigma_h^2 V^2}{MNm_h^2}
          \\&\quad+\frac{M(\sigma_h^2-m_h^2)-\sigma_h^2}{MNm_h^2} \mathbb{E}\left[\left\Vert\nabla J(\boldsymbol{\theta}^k)\right\Vert^2\right].
     \end{split}
     \end{equation}

\section{Proof of Theorem~\ref{thm:convergence}}~\label{app:thm1}
    By Lemma~\ref{lm1}, we have
         \begin{equation}\label{ieq1}
         \begin{split}
             &\left(\mathbb{E}\left[J(\boldsymbol{\theta}^{K})\right] -J(\boldsymbol{\theta}^{*})\right) - \left(J(\boldsymbol{\theta}^{0})-J(\boldsymbol{\theta}^{*})\right)
             \\\leq& \sum_{k=0}^{K-1} \left(-\frac{\alpha m_h}{2}\mathbb{E}\left[\left\Vert \nabla J(\boldsymbol{\theta}^{k})\right\Vert^2 \right]\right.\\&\left.+ \frac{\alpha m_h}{2} \mathbb{E}\left[\left\Vert  \frac{\boldsymbol{v}_k}{m_h N}-\nabla J(\boldsymbol{\theta}^{k}) \right\Vert^2 \right]\right.
            \\&\left. + \left( \frac{L}{2}- \frac{1}{2\alpha m_h} \right)\mathbb{E}\left[\left\Vert \boldsymbol{\theta}^{k+1}-\boldsymbol{\theta}^{k} \right\Vert^2\right]\right)
            \\=& -\frac{\alpha m_h}{2}\sum_{k=0}^{K-1} \mathbb{E}\left[\left\Vert \nabla J(\boldsymbol{\theta}^{k})\right\Vert^2\right] \\&+ \frac{\alpha m_h}{2} \sum_{k=0}^{K-1} \mathbb{E}\left[\left\Vert  \frac{\boldsymbol{v}_k}{m_h N}-\nabla J(\boldsymbol{\theta}^{k}) \right\Vert^2 \right]
            \\&+  \left(\frac{L}{2}- \frac{1}{2\alpha m_h} \right)\sum_{k=0}^{K-1}\mathbb{E}\left[\left\Vert \boldsymbol{\theta}^{k+1}-\boldsymbol{\theta}^{k} \right\Vert^2\right]
            \\\leq&\frac{\alpha\left((M-1)\sigma_h^2-M(N+1)m_h^2\right)}{2MN m_h}\sum_{k=0}^{K-1} \mathbb{E}\left[\left\Vert \nabla J(\boldsymbol{\theta}^{k})\right\Vert^2 \right]\\& + \frac{\alpha m_h K\sigma^2}{2N^2} + \frac{\alpha K \sigma_h^2 V^2}{2MNm_h}. 
         \end{split}
     \end{equation}
     If $\sigma_h^2\leq (N+1)m_h^2 <\frac{M(N+1)}{M-1}m_h^2$, we have
    \begin{equation}
         \begin{split}
             \frac{1}{K}\sum_{k=0}^{K-1} \mathbb{E}\left[\left\Vert\nabla J(\boldsymbol{\theta}^{k})\right\Vert^2 \right]
             \leq& \frac{2MN m_h\left(J(\boldsymbol{\theta}^{0})-J(\boldsymbol{\theta^*})\right)}{\alpha\Lambda_{N,M}^{\sigma_h,m_h} K} \\&+ \frac{M m_h^2 \sigma^2}{ N\Lambda_{N,M}^{\sigma_h,m_h}} + \frac{\sigma_h^2 V^2}{\Lambda_{N,M}^{\sigma_h,m_h}}.
         \end{split}
     \end{equation}

 \section{Proof of Theorem~\ref{thm:convergence2}}~\label{app:thm2}
  By assumption~\ref{asm:lossbound} and assumption~\ref{asm:gradientbound}, we have
    \begin{equation}
        \begin{split}
            \mathbb{E}\left[\left\Vert \nabla J(\boldsymbol{\theta}^{k})\right\Vert^2 \right] &\leq \left\Vert\sum_{t=0}^T tG \gamma^t  \overline{l}\right\Vert^2 \leq \left\Vert G \overline{l}\sum_{t=0}^\infty t\gamma^t \right\Vert^2 =V^2.
        \end{split}
    \end{equation}
    Combining with~\eqref{ieq1}, we have
     \begin{equation}
         \begin{split}
            &\frac{\alpha (M(N+1)m_h^2+\sigma_h^2)}{2MN m_h}\sum_{k=0}^{K-1} \mathbb{E}\left[\left\Vert \nabla J(\boldsymbol{\theta}^{k})\right\Vert^2 \right]
            \\ \leq&\left(J(\boldsymbol{\theta}^{0})-J(\boldsymbol{\theta}^{*})\right) +\frac{\alpha \sigma_h^2KV^2}{2N m_h} + \frac{\alpha \sigma_h^2  KV^2}{2MNm_h} + \frac{\alpha m_h \sigma^2 K}{2N^2} . 
         \end{split}
     \end{equation}
     Therefore, we have
     \begin{equation}
         \begin{split}
            &\frac{1}{K}\sum_{k=0}^{K-1} \mathbb{E}\left[\left\Vert \nabla J(\boldsymbol{\theta}^{k})\right\Vert^2 \right]
            \\\leq&\frac{2MN m_h \left(J(\boldsymbol{\theta}^{0})-J(\boldsymbol{\theta}^{*})\right)}{\alpha K(M(N+1)m_h^2+\sigma_h^2)}
            + \frac{M\sigma_h^2 V^2}{M(N+1)m_h^2+\sigma_h^2}
            \\&+ \frac{\sigma_h^2 V^2}{M(N+1)m_h^2+\sigma_h^2} 
             + \frac{ M m_h^2\sigma^2}{N(M(N+1)m_h^2+\sigma_h^2)}.
         \end{split}
     \end{equation}

\bibliographystyle{unsrt}

\begin{thebibliography}{10}

\bibitem{watkins1992q}
Christopher~JCH Watkins and Peter Dayan.
\newblock Q-learning.
\newblock {\em Machine learning}, 8:279--292, 1992.

\bibitem{sutton1999policy}
Richard~S Sutton, David McAllester, Satinder Singh, and Yishay Mansour.
\newblock Policy gradient methods for reinforcement learning with function
  approximation.
\newblock {\em Advances in neural information processing systems}, 12, 1999.

\bibitem{konda1999actor}
Vijay Konda and John Tsitsiklis.
\newblock Actor-critic algorithms.
\newblock {\em Advances in neural information processing systems}, 12, 1999.

\bibitem{kober2013reinforcement}
Jens Kober, J~Andrew Bagnell, and Jan Peters.
\newblock Reinforcement learning in robotics: A survey.
\newblock {\em The International Journal of Robotics Research},
  32(11):1238--1274, 2013.

\bibitem{shalev2016safe}
Shai Shalev-Shwartz, Shaked Shammah, and Amnon Shashua.
\newblock Safe, multi-agent, reinforcement learning for autonomous driving.
\newblock {\em arXiv preprint arXiv:1610.03295}, 2016.

\bibitem{shao2019survey}
Kun Shao, Zhentao Tang, Yuanheng Zhu, Nannan Li, and Dongbin Zhao.
\newblock A survey of deep reinforcement learning in video games.
\newblock {\em arXiv preprint arXiv:1912.10944}, 2019.

\bibitem{williams1992simple}
Ronald~J Williams.
\newblock Simple statistical gradient-following algorithms for connectionist
  reinforcement learning.
\newblock {\em Machine learning}, 8:229--256, 1992.

\bibitem{baxter2001infinite}
Jonathan Baxter and Peter~L Bartlett.
\newblock Infinite-horizon policy-gradient estimation.
\newblock {\em journal of artificial intelligence research}, 15:319--350, 2001.

\bibitem{papini2018stochastic}
Matteo Papini, Damiano Binaghi, Giuseppe Canonaco, Matteo Pirotta, and Marcello
  Restelli.
\newblock Stochastic variance-reduced policy gradient.
\newblock In {\em International conference on machine learning}, pages
  4026--4035. PMLR, 2018.

\bibitem{silver2014deterministic}
David Silver, Guy Lever, Nicolas Heess, Thomas Degris, Daan Wierstra, and
  Martin Riedmiller.
\newblock Deterministic policy gradient algorithms.
\newblock In {\em International conference on machine learning}, pages
  387--395. Pmlr, 2014.

\bibitem{lillicrap2015continuous}
Timothy~P Lillicrap, Jonathan~J Hunt, Alexander Pritzel, Nicolas Heess, Tom
  Erez, Yuval Tassa, David Silver, and Daan Wierstra.
\newblock Continuous control with deep reinforcement learning.
\newblock {\em arXiv preprint arXiv:1509.02971}, 2015.

\bibitem{nadiger2019federated}
Chetan Nadiger, Anil Kumar, and Sherine Abdelhak.
\newblock Federated reinforcement learning for fast personalization.
\newblock In {\em IEEE Second International Conference on Artificial
  Intelligence and Knowledge Engineering (AIKE)}, pages 123--127, 2019.

\bibitem{qi2021federated}
Jiaju Qi, Qihao Zhou, Lei Lei, and Kan Zheng.
\newblock Federated reinforcement learning: Techniques, applications, and open
  challenges.
\newblock {\em arXiv preprint arXiv:2108.11887}, 2021.

\bibitem{xu2021multiagent}
Minrui Xu, Jialiang Peng, BB~Gupta, Jiawen Kang, Zehui Xiong, Zhenni Li, and
  Ahmed~A Abd El-Latif.
\newblock Multiagent federated reinforcement learning for secure incentive
  mechanism in intelligent cyber--physical systems.
\newblock {\em IEEE Internet of Things Journal}, 9(22):22095--22108, 2021.

\bibitem{khodadadian2022federated}
Sajad Khodadadian, Pranay Sharma, Gauri Joshi, and Siva~Theja Maguluri.
\newblock Federated reinforcement learning: Linear speedup under markovian
  sampling.
\newblock In {\em International Conference on Machine Learning}, pages
  10997--11057. PMLR, 2022.

\bibitem{chen2021communication}
Tianyi Chen, Kaiqing Zhang, Georgios~B Giannakis, and Tamer Ba{\c{s}}ar.
\newblock Communication-efficient policy gradient methods for distributed
  reinforcement learning.
\newblock {\em IEEE Transactions on Control of Network Systems}, 9(2):917--929,
  2021.

\bibitem{yang2020federated}
Kai Yang, Tao Jiang, Yuanming Shi, and Zhi Ding.
\newblock Federated learning via over-the-air computation.
\newblock {\em IEEE Transactions on Wireless Communications}, 19(3):2022--2035,
  2020.

\bibitem{amiri2020federated}
Mohammad~Mohammadi Amiri and Deniz G{\"u}nd{\"u}z.
\newblock Federated learning over wireless fading channels.
\newblock {\em IEEE Transactions on Wireless Communications}, 19(5):3546--3557,
  2020.

\bibitem{zhu2020one}
Guangxu Zhu, Yuqing Du, Deniz G{\"u}nd{\"u}z, and Kaibin Huang.
\newblock One-bit over-the-air aggregation for communication-efficient
  federated edge learning: Design and convergence analysis.
\newblock {\em IEEE Transactions on Wireless Communications}, 20(3):2120--2135,
  2020.

\bibitem{li2019wirelessly}
Xiaoyang Li, Guangxu Zhu, Yi~Gong, and Kaibin Huang.
\newblock Wirelessly powered data aggregation for iot via over-the-air function
  computation: Beamforming and power control.
\newblock {\em IEEE Transactions on Wireless Communications}, 18(7):3437--3452,
  2019.

\bibitem{liu2023over}
Zhiyan Liu, Qiao Lan, Anders~E Kal{\o}r, Petar Popovski, and Kaibin Huang.
\newblock Over-the-air multi-view pooling for distributed sensing.
\newblock {\em arXiv preprint arXiv:2302.09771}, 2023.

\bibitem{park2021optimized}
Pangun Park, Piergiuseppe Di~Marco, and Carlo Fischione.
\newblock Optimized over-the-air computation for wireless control systems.
\newblock {\em IEEE Communications Letters}, 26(2):424--428, 2021.

\bibitem{zhu2021over}
Guangxu Zhu, Jie Xu, Kaibin Huang, and Shuguang Cui.
\newblock Over-the-air computing for wireless data aggregation in massive iot.
\newblock {\em IEEE Wireless Communications}, 28(4):57--65, 2021.

\bibitem{krouka2021communication}
Mounssif Krouka, Anis Elgabli, Chaouki~Ben Issaid, and Mehdi Bennis.
\newblock Communication-efficient and federated multi-agent reinforcement
  learning.
\newblock {\em IEEE Transactions on Cognitive Communications and Networking},
  8(1):311--320, 2021.

\bibitem{dal2023over}
Nicolo Dal~Fabbro, Aritra Mitra, Robert Heath, Luca Schenato, and George~J
  Pappas.
\newblock Over-the-air federated td learning.
\newblock 2023.

\bibitem{sery2020analog}
Tomer Sery and Kobi Cohen.
\newblock On analog gradient descent learning over multiple access fading
  channels.
\newblock {\em IEEE Transactions on Signal Processing}, 68:2897--2911, 2020.

\bibitem{zhao2021distributed}
Xiaoxiao Zhao, Jinlong Lei, Li~Li, and Jie Chen.
\newblock Distributed policy gradient with variance reduction in multi-agent
  reinforcement learning.
\newblock {\em arXiv preprint arXiv:2111.12961}, 2021.

\bibitem{xu2023gradient}
Xing Xu, Rongpeng Li, Zhifeng Zhao, and Honggang Zhang.
\newblock The gradient convergence bound of federated multi-agent reinforcement
  learning with efficient communication.
\newblock {\em IEEE Transactions on Wireless Communications}, 2023.

\bibitem{mordatch2017emergence}
Igor Mordatch and Pieter Abbeel.
\newblock Emergence of grounded compositional language in multi-agent
  populations.
\newblock {\em arXiv preprint arXiv:1703.04908}, 2017.

\end{thebibliography}

\end{document}